\documentclass{article}

\usepackage[T1]{fontenc}
\usepackage[latin9]{inputenc}
\usepackage{subfigure}

\usepackage{amsmath}
\usepackage{amsthm}
\usepackage{amssymb}
\usepackage{graphicx}
\usepackage{xcolor}         

\newtheorem{thm}{Theorem}

\newtheorem{lem}{Lemma}

\usepackage[preprint]{neurips_2021}

\title{Momentum via Primal Averaging: Theoretical Insights and Learning Rate Schedules for Non-Convex Optimization}

\author{%
  Aaron Defazio\\
  Facebook AI Research\\
  New York
}

\begin{document}
\maketitle

\begin{abstract}
Momentum methods are now used pervasively within the machine learning
community for training non-convex models such as deep neural networks.
Empirically, they outperform traditional stochastic gradient descent
(SGD) approaches. In this work we develop an Lyapunov analysis of SGD
with momentum (SGD+M), by utilizing a equivalent rewriting of the
method known as the stochastic primal averaging (SPA) form. This analysis is tight enough to give precise insights into when SGD+M may outperform
SGD, and what hyper-parameter schedules will work and why.
\end{abstract}

\section{Introduction}

Heavy ball methods have a long history dating back to the work of
\citet{Polyak64}. More recently, the stochastic heavy ball method,
also known as stochastic gradient descent with momentum (SGD+M), has
become a standard for deep learning practitioners since it was observed
that momentum significantly helps on common computer vision problems
\citep{Sutskever13}. 

In this work we provide an analysis of SGD+M for non-convex problems
that is much tighter than past approaches. The form of this analysis
is tight enough to provide several insights into the practical behavior
of SGD+M, including suggesting hyper-parameter schemes and indicating
why SGD+M is faster than SGD at the early stages of optimization.
We believe our analysis technique is also useful in it's own right,
and may be a good starting point for analyzing other methods that
involve momentum.

There is a substantial body of prior work on the SGD+M method. Non-asymptotic
convergence in the non-stochastic convex setting was first established
by \citet{Ghadimi2014}, where it is shown that for parameters of
the form $\beta_{k}=k/(k+2)$ and $\alpha_{k}\propto1/(k+2)$, the
method obtains last iterate convergence rates comparable to gradient
descent. They also show that when $\beta_{k}$ is constant the best
convergence rate they are able to obtain is worse than gradient descent
by a constant factor $\beta$. Unfortunately their proof technique
does not extend readily to the stochastic setting. \citet{pmlr-v40-Flammarion15}
consider both momentum and accelerated methods for convex quadratic
problems, where they are able to establish bounds using the technique
of difference equations, even with noisy (but not stochastic) gradients.

\citet{yuanyingsayed2016} analyze momentum methods under the assumption
of strong convexity and small step sizes in the online setting, and
show no actual advantage to momentum methods in this setting. \citet{Can19}
establish strong results in another special case, where gradient noise
is bounded and the objective is either strongly convex or quadratic.
\citealt{NeedellWardSrebro2014} also consider the strongly-convex
case, using proof techniques developed for the randomized Kaczmarz
algorithm. Also under a quadratic assumption, \citet{pmlr-v75-jain18a}
analyzed an accelerated scheme related to Nesterov's accelerated method
in the stochastic case. While the heavy ball method is known to provide
accelerated convergence rates for quadratic problems, these rates
provably do not extend to the non-quadratic case \citep{Kidambi18}.

\citet{yanyan2018} provide the first analysis of momentum (with
an earlier preprint \citealp{Yang16}), including Nesterov's scheme,
in the non-convex case, establishing a bound of the form:
\begin{align*}
 & \min_{k=0,\dots,t}\mathbb{E}\left[\left\Vert \nabla f(x_{k})\right\Vert ^{2}\right]\\
 & \leq\frac{2\left[f(x_{0})-f_{*}\right](1-\beta)}{t+1}\max\left\{ \frac{2L}{1-\beta},\frac{\sqrt{t+1}}{C}\right\}
 +\frac{CL\beta^{2}\left(G^{2}+\sigma^{2}+L\sigma^{2}\left(1-\beta\right)^{2}\right)}{\sqrt{t+1}\left(1-\beta\right)^{3}},
\end{align*}
where $\left\Vert \nabla f(x)\right\Vert \leq G$, $\mathbb{E}\left[\left\Vert \nabla f(x,\xi)-\nabla f(x)\right\Vert ^{2}\right]\leq\sigma^{2}$,
$C$ is a positive constant, and f is $L$-Lipschitz smooth, for method Eq. \ref{eq:standard-mom}.
This rate is much looser than the rate we establish in this work,
and our rate includes no unspecified constants. \citet{yujinyang2020}
consider the distributed non-convex setting, where they establish
a rate that is also looser than our own. A general result of almost-sure
convergence is shown by \citet{Gadat18} in the non-convex setting.

Recently, \citet{sebbouh2020convergence} establish rates for the
convex and strongly convex settings in the stochastic case that mirror
the tight rates in the deterministic case of \citet{Ghadimi2014},
using a Lyapunov function analysis. Along with \citealp{primal_averaging}
and \citealp{defazio2020factorial}, this line of work shows that
the primary advantage of the heavy ball method over SGD is that it
it is possible to show tight convergence of the last-iterate, rather
than an average of iterates (as for SGD). Last-iterate convergence
rates for SGD are weaker than the average iterate convergence unless
very careful parameter schemes are used \citep{Jain19}, and even
then only when the stopping time is known in advance.

For the non-convex setting, the closest work to ours is that of \citet{liu2020improved},
who use a Lyapunov analysis and make use of the same $z_{k}$ quantity
that we use in this work, as an ancillary point. In our view $z_{k}$
should be a key part of the algorithm, rather than a derived quantity.
They give the following bound on their Lyapunov function $\Lambda_k$:
\begin{align*}
 & \mathbb{E}[\Lambda_{k+1}]-\Lambda_{k}\\
 & \leq\left(-\alpha+\frac{-\beta+\beta^{2}}{2(1-\beta)}L\alpha^{2}+4c_{1}\alpha^{2}\right)\mathbb{E}\left[\left\Vert g_{k}\right\Vert ^{2}\right]
 +\frac{\beta^{2}}{2(1+\beta)}L\alpha^{2}\sigma^{2}+\frac{1}{2}L\alpha^{2}\sigma^{2}+2c_{1}\frac{1-\beta}{1+\beta}\alpha^{2}\sigma^{2}
\end{align*}
where $\Lambda_{k}=f(z_{k})-f^{*}+\sum_{i=1}^{k-1}c_{i}\left\Vert x^{k+1-i}-x^{k-i}\right\Vert ^{2}.$
We refer the reader to their paper for details in the values of $c$,
$\alpha$ and the settings in which this bound holds. This bound is
looser than the one we derive, and provides less insight into the
practical behavior of SGD+M than the bound we derive in this work.
In other work on the non-convex case, \citet{normalizedsgd2020} analyze
a form of SGD+M with normalized steps. The recent work of \citet{weakconvexity2020}
analyze SGD+M under a weak convexity assumption as well as in the
smooth case, using different proof techniques than we explore in this
work, resulting in a looser bound.

\section{The averaging form of momentum}
The stochastic gradient method with momentum (SGD+M) is commonly written
in the following form:
\begin{align}
m_{k+1} & =\beta_{k}m_{k}+\nabla f(x_{k},\xi_{k}),\nonumber \\
x_{k+1} & =x_{k}-\alpha_{k}m_{k+1},\label{eq:standard-mom}
\end{align}
where $x_{k}$ is the iterate sequence, and $m_{k}$ is the momentum
buffer, and $\nabla f(x_{k},\xi_{k})$ the stochastic gradient at
step $k$. For our analysis we will not use this form, instead, we
will make use of the recently discovered averaging form of the momentum
method \citep{adefazio-curvedgeom2019,sebbouh2020convergence}, also
discovered as a separate method (without relating to SGD+M) under
the name SPA (stochastic primal averaging) by \citet{primal_averaging}:
\begin{align*}
z_{k+1} & =z_{k}-\eta_{k}\nabla f\left(x_{k},\xi_{k}\right),\\
x_{k+1} & =\left(1-c_{k+1}\right)x_{k}+c_{k+1}z_{k+1}.
\end{align*}
For specific choices of values for the hyper-parameters, the $x_{k}$
sequence generated by this method will be identical to that of SGD+M.
The quantity $z_{k}$ is actually used in some early analysis of momentum
methods, but without this explicit transformation \citep{Ghadimi2014}.
A continuous time version of this update is analyzed in \citet{krichene2016},
but without relating it to the heavy ball method. 

The averaging form, compared to the standard form, appears to be easier
to analyze theoretically, as the $z$ sequence arises naturally when
performing a Lyapunov-style analysis of the method. The mapping between
the two forms is described in the following theorem.
\begin{thm}
\label{thm:correspondence}The $x_{k}$ sequences of the SPA method
and SGD+M are equal when $z_{0}=x_{0}$ and for all $k\geq0$:
\[
\eta_{k+1}=\frac{\eta_{k}-\alpha_{k}}{\beta_{k+1}},\qquad c_{k+1}=\frac{\alpha_{k}}{\eta_{k}},
\]
conversely, $\alpha_{k}=\eta_{k}c_{k+1},$ and $\beta_{k}=\frac{\eta_{k-1}}{\eta_{k}}\left(1-c_{k}\right).$
\end{thm}
This correspondence results in surprising dynamics when otherwise
reasonable hyper-parameter schedules are mapped from one form to another.
For illustration, we will consider the case where one or both of the
parameters are changed by a fixed factor, as is commonly done when
using a stage-wise schedule. We apply this change at step 20 of 100
steps, with $\beta=0.9$ and $\alpha=1.0$. Each case is shown in
Figure \ref{fig:correspondences}.
\begin{description}
\item [{(a)}] When the learning rate $\alpha$ of the SGD+M form is decreased
by a fixed factor while $\beta$ is kept constant, the learning rate
in the SPA form begins to grow geometrically, and $c$ shrinks geometrically.
This is the most common schedule used in practice for the SGD+M method,
and the fact that it causes such odd behavior in the SPA form is a
cause for concern. This schedule in SPA form is NOT supported by our
Lyapunov analysis.
\item [{(b)}] When the momentum constant $\beta$ is changed (in our example
from 0.9 to 0.8), while keeping $\alpha$ constant, a similar geometric
increase/decrease behavior occurs as in case 1.
\end{description}
Both behaviors above are unsatisfying when viewed from the perspective
of the SPA method. We may also perform the reverse operation, and
consider the behavior of the hyper-parameters of the SGD+M method
when step-wise schedules are used for the SPA form.
\begin{description}
\item [{(c)}] When $\eta_{k}$ is decreased 10 fold, a spike occurs in
$\beta_{k}$, after which $\alpha_{k}$ drops 10 fold and $\beta_{k}$
drops back to it's earlier value.
\item [{(d)}] When $c_{k}$ is increased 10 fold, then the SGD+M form is
better behaved, as $\alpha_{k}$ increases 10 fold and $\beta_{k}$
drops to 0. This is reasonable behavior as this change corresponds
to removing the momentum in both forms, while attempting to keep the
effective step size the same.
\item [{(e)}] As we show in Section \ref{subsec:insight-reduce-averaging},
the most theoretically motivated choice is to actually change both
$\eta_{k}$ and $c_{k}$. This unfortunately also results in a spike
in $\alpha_{k}$
\item [{(f)}] Replacing the sudden change in $\eta_{k}$ and $c_{k}$ by
a gradual change removes the spike and keeps $\beta_{k}$ below $1$.
We show in Section \ref{subsec:insight-no-abrupt-changes} that a
gradual change is actually required by our Lyapunov theory.
\end{description}

\begin{figure}
\centering
SGD+M to SPA
\subfigure[$\alpha_{k}$ decreased]{
\includegraphics{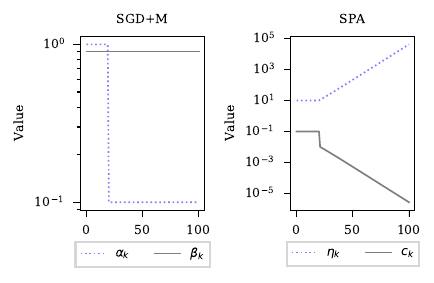}
}\subfigure[$\beta_{k}$ decreased]{
\includegraphics{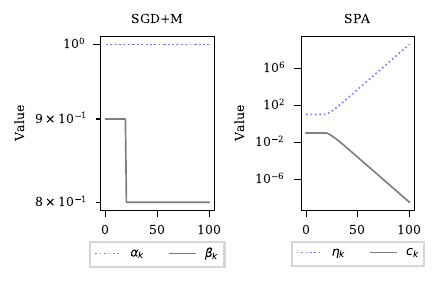}
}

\vspace{1em}
\rule[0.5ex]{1\textwidth}{1pt}
\vspace{1em}

SPA to SGD+M
\subfigure[$\eta_{k}$ decreased]{
\includegraphics{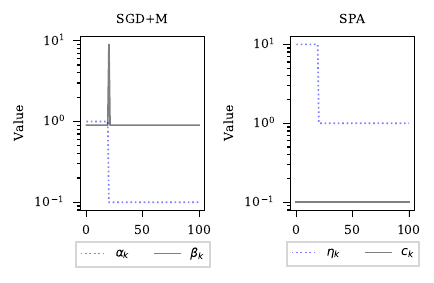}
}\subfigure[$c_{k}$ increased ]{
\includegraphics{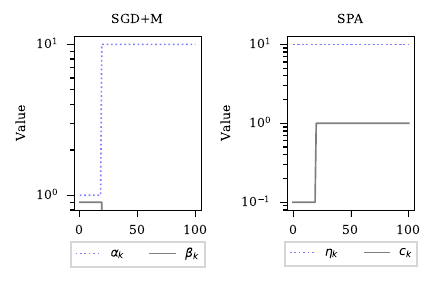}
}
\subfigure[$\alpha_{k}$ decreased and $c_{k}$ increased]{
\includegraphics{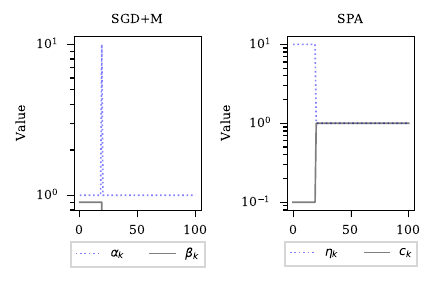}
}\subfigure[\label{fig:gradual-schematic}$\alpha_{k}$ decreased and $c_{k}$
increased gradually]{
\includegraphics{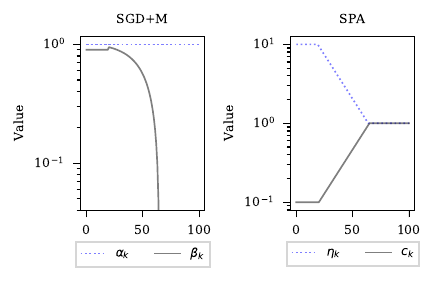}
}\caption{\label{fig:correspondences}The behavior of the hyper-parameters of
the SPA form when they are set so as to maintain an identical iterate
sequence as the SGD+M form, and vice-versa.}
\end{figure}

\section{Lyapunov analysis }

In the Lyapunov analysis technique, a non-negative function $\Lambda_{k}=\Lambda(x_{0:k},z_{0:k},\dots)$
is defined in terms of all indexed quantities in the algorithm up
to the current time-step, for the purposes of controlling the convergence
of the optimization method under analysis. In the convex case, the
standard approach is to show that $\mathbb{E}\left[f(x_{k})-f_{*}\right]\leq\Lambda_{k}-\mathbb{E}\left[\Lambda_{k+1}\right]+\text{noise}$,
after which we can apply a telescoping argument to complete the proof.
In the non-convex case we instead attempt to control the norm of the
gradient of $f$, through a bound of the form:
\[
d_{k}\left\Vert \nabla f(x_{k})\right\Vert ^{2}\leq\Lambda_{k}-\mathbb{E}\left[\Lambda_{k+1}\right]+\text{noise}
\]
where $d_{k}$ is some constant, and with expectations over randomness
in the current step $k$, conditional on prior steps (we use this
convention in the remainder of this work). We call an equation of
this form a Lyapunov step equation. In the case of SGD it is straight-forward
to show that the Lyapunov step takes the following form, assuming $\mathbb{E}\left[\left\Vert \nabla f(x_{k},\xi_{k})\right\Vert ^{2}\right]\leq G^{2}$) and that $f$ is $L$-Lipschitz smooth:
\begin{equation}
\frac{1}{\eta_{k}}\mathbb{E}\left[\left\Vert \nabla f(x_{k})\right\Vert ^{2}\right]\leq\Lambda_{k}-\mathbb{E}\left[\Lambda_{k+1}\right]+\frac{1}{2}LG^{2}+R_{k},\label{eq:sgd-lya-step}
\end{equation}
where $\Lambda_{k}=\eta_{k}^{-2}\mathbb{E}\left[f(x_{k})-f_{*}\right]$
and $R_{k}=(\eta_{k}^{-2}-\eta_{k-1}^{-2})\left[f(z_{k})-f_{*}\right]$.
From this Lyapunov step equation, a standard telescoping argument
(we give details in the appendix) completes the convergence rate proof,
yielding a bound on $\mathbb{E}\left[\left\Vert \nabla f(x_{i})\right\Vert ^{2}\right]$
for a randomly sampled $i$.

\subsection{Momentum case}
In the appendix, we construct the following Lyapunov function $\Lambda$
for the SGD+M method in SPA form:
\begin{align}
\Lambda_{k+1} & =\frac{1}{\eta_{k}^{2}}\left[f(z_{k+1})-f_{*}\right]+\frac{L}{\eta_{k}}\left(\frac{1}{c_{k}}-1\right)\left[f(x_{k})-f_{*}\right] + \frac{L}{2\eta_{k}^{2}c_{k+1}^{2}}\left\Vert x_{k+1}-x_{k}\right\Vert ^{2} \label{eq:nonstatic-lya} 
\end{align}
The Lyapunov step equation for $k\geq1$, with expectations conditioning
on $x_{k}$ and prior gradients $\text{\ensuremath{\nabla f(x_{i}}) for \ensuremath{i\leq k}}$
is:
\begin{align}
 & \frac{1}{2\eta_{k}}\left\Vert \nabla f(x_{k})\right\Vert ^{2}+\frac{1}{2\eta_{k}}\left\Vert \nabla f(z_{k})\right\Vert ^{2}\nonumber \\
 & \leq\Lambda_{k}-\mathbb{E}\left[\Lambda_{k+1}\right]+L\mathbb{E}\left[\left\Vert \nabla f(x_{k},\xi_{k})\right\Vert ^{2}\right]+R_{k}\nonumber \\
 & +\frac{1}{2}\left[\frac{1}{\eta_{k}^{2}}\left(\frac{1}{c_{k}}-1+\eta_{k}L\right)\left(\frac{1}{c_{k}}-1\right) 
 +\frac{1}{\eta_{k}}L\left(\frac{1}{c_{k}}-1\right)^{2}-\frac{1}{\eta_{k-1}^{2}c_{k}^{2}}\right]L\left\Vert x_{k}-x_{k-1}\right\Vert ^{2}.\label{eq:lyapunov-step}
\end{align}
where the remainder term $R_{k}$ is defined as:
\begin{align*}
R_{k} & ={\textstyle \left[\frac{L}{\eta_{k}}\left(\frac{1}{c_{k}}-1\right)-\frac{L}{\eta_{k-1}}\left(\frac{1}{c_{k-1}}-1\right)\right]\left[f(x_{k-1})-f_{*}\right]} +{\textstyle \left[\frac{1}{\eta_{k}^{2}}-\frac{1}{\eta_{k-1}^{2}}\right]\left[f(z_{k})-f_{*}\right]}
\end{align*}
This bound is our key theoretical result. We give the full telescoped proof using this bound in the appendix
yielding a $\mathcal{O}(k^{-1/2})$ rate. The key differences between
this bound and the bound for SGD (Equation \ref{eq:sgd-lya-step})
are:
\begin{enumerate}
\item The convergence rate is in terms of $\frac{1}{2\eta_{k}}\left\Vert \nabla f(z_{k})\right\Vert ^{2}+\frac{1}{2\eta_{k}}\left\Vert \nabla f(x_{k})\right\Vert ^{2}$
for SGD+M compared to $\frac{1}{\eta_{k}}\left\Vert \nabla f(x_{k})\right\Vert ^{2}$
for SGD. When we telescope to give a convergence rate bound, the bound
is on a randomly sampled iterate from a weighted set of $x_{k}$ and
$z_{k}$ rather than just $x_{k}$.
\item There is an extra $\left\Vert x_{k}-x_{k-1}\right\Vert ^{2}$ term
on the right which will be negative and hence beneficial for typical
choices of the hyper-parameters, as we show in Section \ref{subsec:insight-no-abrupt-changes}.
\item The noise term $\mathbb{E}\left[\left\Vert \nabla f(x_{k},\xi_{k})\right\Vert ^{2}\right]$
is weighted by $L$ for SGD+M and $\frac{1}{2}L$ for SGD. Although
this noise term is twice as large for SGD+M, we show in Section \ref{subsec:noise-cancelation},
that almost half of it is canceled by the negative $\left\Vert x_{k}-x_{k-1}\right\Vert ^{2}$
term when additional assumptions are made, meaning that the noise
is actually essentially the same as SGD.
\item The Lyapunov function of SGD is just $\eta_{k}^{-2}f(x_{k})$, whereas
the Lyapunov function of SGD+M involves $\eta_{k}^{-2}f(z_{k})$ plus
two other terms. After telescoping for $T$ steps (as we show in the
appendix), the $\left\Vert x_{k+1}-x_{k}\right\Vert ^{2}$ term drops
out, and the $\left[f(x_{k})-f_{*}\right]$ term decays at a rate
$\sqrt{T}$ faster than the other terms, making it negligible at the
end of optimization for typical values of $c_{k}$, i.e. when $\left(\frac{1}{c_{1}}-1\right)\ll\sqrt{T}$.
These terms appear to be the main limiting factor for how small $c_{k}$
can be chosen (i.e. how much momentum is used).
\item The $R_{k}$ term is 0 when $\eta_{k}=\eta_{k-1}$ and $c_{k}=c_{k-1}$,
otherwise it contains an ``error'' accumulated from changing the
hyper-parameters. In a stage-wise hyper-parameter scheme this error
accumulation happens only at the end of each stage, and it's contribution
to the final convergence rate bound will be weighted with $1/T$, significantly
smaller than the $1/\sqrt{T}$ weight of the primary terms. This is
similar behavior to the $R_{k}$ term in the SGD step equation.
\end{enumerate}

\section{Insight \#1: Momentum may cancel out noise during early iterations}

\label{subsec:noise-cancelation} The noise term in the Lyapunov step
of SGD+M is twice as large as the noise term $\frac{1}{2}L\mathbb{E}\left[\left\Vert \nabla f(x_{k},\xi_{k})\right\Vert ^{2}\right]$
in SGD. Although typically such small differences are disregarded
in the analysis of optimization methods, in this case we believe that
this term gives substantial insight into the practical behavior of
the two methods. The difference between the bounds on the convergence
rate of the two methods will depend crucially on the magnitude of
the negative $\left\Vert x_{k}-x_{k-1}\right\Vert ^{2}$ term in comparison
to this noise term. When this negative iterate difference term is
sufficiently large, SGD+M can be expected to converge faster than
SGD. In this section we analyze this term in detail. We will assume
in this section that $c_{k}=c$ and $\eta_{k}=\eta$ are independent
of $k$, we consider in Section \ref{subsec:insight-no-abrupt-changes}
what happens to $\left\Vert x_{k}-x_{k-1}\right\Vert ^{2}$ when they
change in a step-wise scheme.

Firstly note that the the weight of $\left\Vert x_{k}-x_{k-1}\right\Vert ^{2}$
in the Lyapunov step (\ref{eq:lyapunov-step}) can be written in the
following form after expanding and simplifying when using constant
hyper-parameters:
\[
\frac{L}{2}\left[-\frac{2}{\eta^{2}c}+\frac{1}{\eta^{2}}+\frac{L}{\eta c^{2}}-\frac{L}{\eta c}\right].
\]
To understand the magnitude of $\left\Vert x_{k}-x_{k-1}\right\Vert ^{2}$,
we may consider it's recursive expansion:
\begin{align}
  \left\Vert x_{k}-x_{k-1}\right\Vert ^{2} & = \left(1-c\right)^{2}\left\Vert x_{k-1}-x_{k-2}\right\Vert ^{2}\nonumber\\
 & +c^{2}\eta^{2}\left\Vert \nabla f\left(x_{k-1},\xi_{k-1}\right)\right\Vert ^{2} -2\eta c^{2}\left(\frac{1}{c}-1\right)\left\langle \nabla f(x_{k-1}),x_{k-1}-x_{k-2}\right\rangle .\label{eq:unwinding}
\end{align}
This recursive expression may be further unwound, giving a geometrically
decreasing weighted sequence. We consider the inner-product term in
the next section, for the moment we assume that it has expectation
zero. The gradient term $\left\Vert \nabla f\left(x_{k-1},\xi_{k-1}\right)\right\Vert ^{2}$
here gives some insight into why we may expect cancelation against
the noise term in the Lyapunov step. When this expression is unwound,
it contains a contribution from all past gradients:
\[
\sum_{i=0}^{k}\left(1-c\right)^{2i}c^{2}\eta^{2}\mathbb{E}\left[\left\Vert \nabla f\left(x_{k-i},\xi_{k-i}\right)\right\Vert ^{2}\right],
\]
So the noise term $\frac{1}{2}L\mathbb{E}\left[\left\Vert \nabla f(x_{k},\xi_{k})\right\Vert ^{2}\right]$
is not canceled immediately by the negative iterate distance $\left\Vert x_{k}-x_{k-1}\right\Vert ^{2}$,
instead, it cancels part of the noise from past iterations. In fact,
we can see that after some step $i$, the noise term introduced by
that step over and above SGD, namely $\frac{1}{2}L\mathbb{E}\left[\left\Vert \nabla f(x_{i},\xi_{i})\right\Vert ^{2}\right]$
will be partially negated at every successive step, in a geometrically
decaying fashion. Considering it as an infinite sum, we have:
\begin{align*}
\sum_{i=0}^{\infty}\left(1-c\right)^{2i} & =\frac{1}{1-(1-c)^{2}}=\frac{1}{c\left(2-c\right)}
\end{align*}
Is this sufficient for the negative terms to cancel the additional
noise over SGD? Let's consider the weight heuristically before providing
a more precise argument. Firstly, consider the weight in front of
$\left\Vert x_{k}-x_{k-1}\right\Vert ^{2}$ . The dominating term
in this expression for small $\eta$ and $c$ is $-L/\eta^{2}c$.
The $\left\Vert \nabla f(x_{i},\xi_{i})\right\Vert ^{2}$ term is
multiplied by $c^{2}\eta^{2}$ in the geometric sum. The infinite
sum is above is $1/2c$ for small $c$, so we find that we have:
\[
-\frac{L}{\eta^{2}c}\cdot c^{2}\eta^{2}\cdot\frac{1}{2c}=\frac{L}{2},
\]
which is exactly large enough to cancel the additional noise. We can
make this argument precise using the tools of Lyapunov analysis, without
requiring the above simplifications. In particular, we can augment
the Lyapunov function with an additional term:
\[
\frac{L}{2\eta c^{2}}\left[\frac{L\left(1-c\right)}{c\left(2-c\right)}-\frac{1}{\eta}\right]\left\Vert x_{k+1}-x_{k}\right\Vert ^{2}.
\]
As we shown in the appendix, as long as $\eta\leq\frac{2c\left(2-c\right)}{L(1-c)},$
this term captures the additional noise introduced at each step $(k=i)$,
and how it decays geometrically overtime. With the addition of this
term in the Lyapunov function, the noise term reduces to \[
\left(1+\frac{\eta L\left(1-c\right)}{c\left(2-c\right)}\right)\frac{L}{2}\mathbb{E}\left[\left\Vert \nabla f(x_{i},\xi_{i})\right\Vert ^{2}\right],
\]
almost matching SGD except for the term $\frac{\eta L\left(1-c\right)}{c\left(2-c\right)},$which
is very small for the $\eta\propto1/(L\sqrt{T)}$ values that the
theory supports. Note however that by expanding $\left\Vert x_{k}-x_{k-1}\right\Vert ^{2}$
we must also consider the additional inner-product terms introduced
in Eq. \ref{eq:unwinding}, which we do in the next section.

\paragraph{When momentum helps}

By expanding the recursive definition of $\left\Vert x_{k}-x_{k-1}\right\Vert ^{2}$,
we have halved the noise term, but at the expense of introducing an
inner-product term proportional to 
\[
-2\eta c^{2}\left(\frac{1}{c}-1\right)\sum_{i=0}^{k}\left(1-c\right)^{2i}\left\langle \nabla f\left(x_{i}\right),x_{i}-x_{i-1}\right\rangle .
\]
\sloppy This term gives a precise characterization of when the convergence
rate bound for SGD+M will be tighter than SGD; when for a particular
weighted average, each $\nabla f\left(x_{i-1}\right)$ is on average
positively aligned or at worst orthogonal to the momentum buffer:
$m_{i-1}\propto-\left(x_{i-1}-x_{i-2}\right)$. If on average they
are highly positively correlated, then we can expect momentum methods
to significantly outperform non-momentum methods. 

The correlation between the momentum buffer and the next gradient
is not assured during optimization. Intuitively, a high correlation
can be expected when the optimization path is heading in a steady
direction, rather than oscillating around a minima or valley. This
is particularly the case in the early stages of optimization, where
there is a clear descent direction, in contrast to the later stages
of optimization, where the optimization path will typically bounce
around a minima or valley due to the noise introduced by using stochastic
gradients. When the optimization path bounces around significantly,
we would expect this inner-product term to be close to zero in expectation.
So although the worst case behavior of SGD+M the convergence rate
bound has double the noise of SGD, in practice we expect a behavior
where at the early stages of optimization it may be faster, and at
the later stages of optimization it will converge at the same rate
as it enters a more noise dominated regime.

\paragraph{An empirical study}
\begin{figure*}
\begin{centering}
\includegraphics[width=0.9\textwidth]{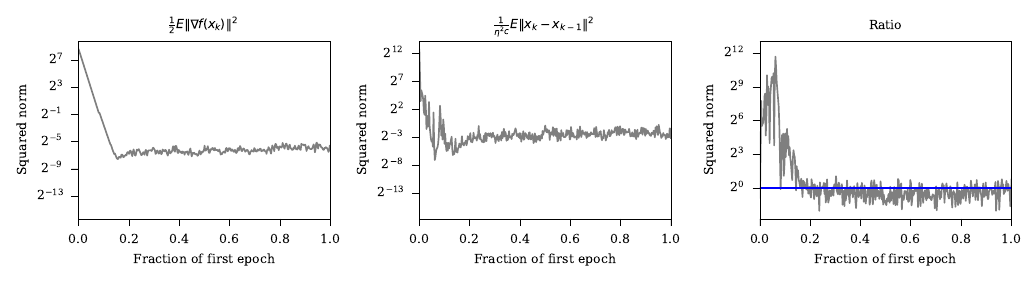}\vspace{-0.5em}
\par\end{centering}
\begin{centering}
\includegraphics[width=0.9\textwidth]{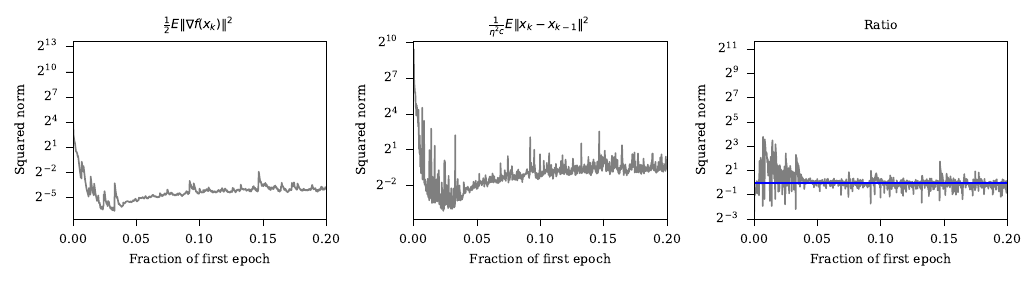}\vspace{-1em}
\par\end{centering}
\caption{{\label{fig:cancelation}Quantities shown are during
CIFAR10 (top row) and ImageNet (bottom row) training with momentum
0.9. Full details of the experimental setup are in the appendix. The
extra negative $x_{k}-x_{k-1}$ term cancels out the large gradient
norm squared term when the shown ratio (right) is above 1. Here this
occurs for the initial steps during the first epoch of training.}}
\end{figure*}
This result also suggests that momentum may ONLY be useful during
the very earliest iterations. In the case of the CIFAR10 \citep{cifar10} problem shown,
it appears to only provide a positive benefit for less than half of
the first epoch, and the benefit is even shorter for ImageNet \citep{imagenet}. To
test this hypothesis, we did a comparison where we turned off momentum
after the first epoch. As shown in Figure \ref{fig:dropm}, this gives
the same test error curve and final test error as for when momentum
is used for the whole run.
\begin{figure}
\begin{centering}
\includegraphics[width=0.3\columnwidth]{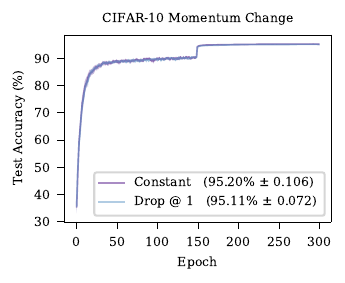} \hspace{4em}\includegraphics[width=0.3\columnwidth]{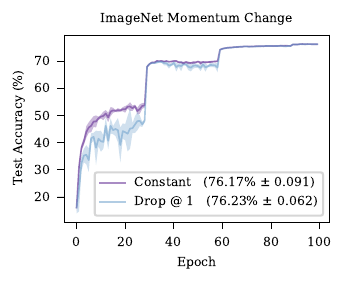}\vspace{-1.5em}
\par\end{centering}
\caption{\label{fig:dropm}Removing momentum by setting $c_{k}=0$
after the first epoch has no negative consequences on the final test
accuracy for either problem. Experimental setup detailed in Section \ref{subsec:insight-no-abrupt-changes}.}
\end{figure}
Our theory suggests that we may directly measure when momentum is
having a positive effect on convergence by comparing the expectations
of the quantities $\frac{1}{\eta^{2}c}\left\Vert x_{k}-x_{k-1}\right\Vert ^{2}$
to $\frac{1}{2}\left\Vert \nabla f(x_{k})\right\Vert ^{2}$. Figure
\ref{fig:cancelation} shows the magnitudes of these two quantities
(smoothed using an exponential moving average to approximate the expectation),
as well as the ratio on two test problems. When considering the ratio,
the $\left\Vert x_{k}-x_{k-1}\right\Vert ^{2}$ term is significantly
bigger at the earliest stages of optimization, and then quickly approaches
the ``noise'' level of 1, corresponding to the inner-product discussed
above being on average $0$. Interestingly, the gradient norm is also
very large during these early iterations, which may explain why momentum
helps so much: It negates the contribution of the noise term to the
convergence rate bound during the iterations when it is largest.
\vspace{-0.5em}
\section{Insight \#2: Reduce $c_{k}$ when you decrease $\eta_{k}$}
\vspace{-0.5em}\label{subsec:insight-reduce-averaging}
\begin{figure*}\vspace{-0.5em}
\begin{centering}
\includegraphics[width=0.3\columnwidth]{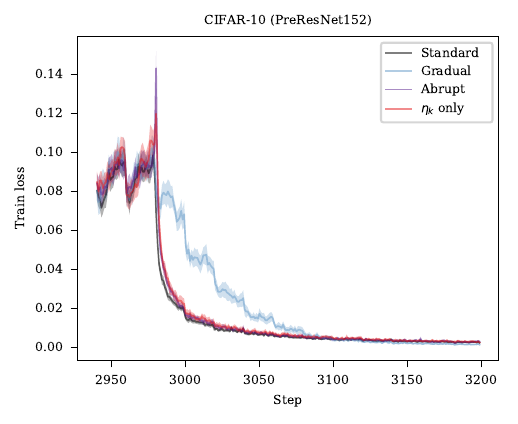}
\includegraphics[width=0.3\columnwidth]{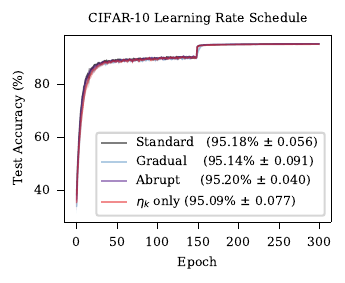}\includegraphics[width=0.3\columnwidth]{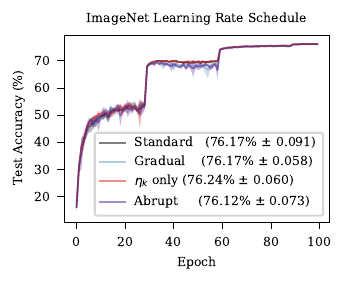}\vspace{-1em}
\par\end{centering}
\caption{Left: \label{fig:gradual_train_loss}training loss
before and after then annealing point where the learning rate is decreased
by a factor 10. Right: \label{fig:gradual_test_accuracy} A comparison of the standard SGD+M, versus primal averaging form with a $\eta_k$ only schedule, versus both $c_k$ and $\eta_k$ decreased either abruptly at 30,60 \& 90 epochs or gradually.}\vspace{-0.5em}
\end{figure*}
Consider the remainder term $R_{k}$:
\begin{align*}
R_{k} & ={\textstyle \left[\frac{L}{\eta_{k}}\left(\frac{1}{c_{k}}-1\right)-\frac{L}{\eta_{k-1}}\left(\frac{1}{c_{k-1}}-1\right)\right]\left[f(x_{k-1})-f_{*}\right]} + {\textstyle \left[\frac{1}{\eta_{k}^{2}}-\frac{1}{\eta_{k-1}^{2}}\right]\left[f(z_{k})-f_{*}\right]}.
\end{align*}
 This term contains the additional error accumulated when the step
size is changed. Our hyper-parameter choices should aim to keep this
term small if possible. The second term involving $\left[f(z_{k})-f_{*}\right]$
is exactly the remainder term that appears in SGD theory, and so we
would not expect to be able to control it further. The first line
involves both $c$ and $\eta$, and so we have a degree of control
over it. We are particularly interested in stage-wise schemes, where
at a certain time-step $T$ the step-size $\eta$ is divided by a
factor $\phi$ (typically 10), i.e. $\eta_{T}=\eta_{T-1}/\phi$. In
that case, we may keep the first term's coefficient at 0 if we choose
parameters satisfying:
\[
\frac{1}{c_{T}}=1+\frac{1}{\phi}\left(\frac{1}{c_{T-1}}-1\right).
\]
For small $c$, this is approximately $c_{T}=\phi c_{T-1}$. I.e.
when the step size is decreased by a factor $\phi$, we should increase
$c$ by that same $\phi$ factor. Using the equivalence in Theorem
\ref{thm:correspondence}, we can see that when constant step sizes
are used, the equivalence:
\[
\beta=\left(1-c\right),\qquad\alpha=\eta c,
\]
suggests that decreasing $\eta$ and increasing $c$ proportionally
actually leaves the step size $\alpha$ the same, but decreases the
amount of momentum $\beta$ in the SGD+M form. This suggests an alternative
approach to the learning rate schedule, when working in SGD+M form:
Decrease $\beta$ rather than decrease $\alpha$, up to the point
where $\beta=0$, corresponding to SGD without momentum.

Unfortunately, this scaling still presents problems, as we see in
Figure \ref{fig:correspondences}, there is an instantaneous spike
in $\alpha_{k}$ when using this approach. Changing the learning rate by 
a large factor suddenly also effects the constants in front of the $\left\Vert x_{k}-x_{k-1}\right\Vert ^{2}$
term in the Lyapunov step, resulting in this term being positive,
rather than negative. We explore this difficulty and a potential solution
in the next section.

\section{Insight \#3: Change hyper-parameters gradually}
\vspace{-0.5em}
\label{subsec:insight-no-abrupt-changes} When constant momentum and
step sizes are used, the weight of the term $\left\Vert x_{k}-x_{k-1}\right\Vert ^{2}$
in the Lyapunov step is non-positive for values of $\eta$ larger
than the typical $2/L$ maximum required for non-momentum methods:
\begin{equation}
\eta\leq\frac{2-c}{Lc(1-c)}.\label{eq:hyper-req}
\end{equation}
However, when $\eta$ changes abruptly by large amounts between steps,
this expression can not be satisfied. Instead, lets determine the
largest multiplicative change in $\eta$ allowed between steps. Let
$\eta_{k}=\eta_{k-1}/r$, where we expect $r$ to be larger than 1.
We use $\eta$ to denote $\eta_{k-1}$ to simplify the notation. We
also apply $\eta L\leq1$ to simplify. This gives:
\[
\frac{r^{2}}{\eta^{2}}\left(\frac{1}{c^{2}}-\frac{1}{c}\right)+\frac{r}{\eta}L\left(\frac{1}{c}-1\right)^{2}\leq\frac{1}{\eta^{2}c^{2}},
\]
Therefore $r^{2}-r^{2}c+r\eta L\left(1-c\right)^{2}-1\leq0$.
Solving this quadratic equation gives two roots, one of which is always
negative, the other root is:
\begin{align*}
r & =\frac{-\eta L\left(1-c\right)^{2}+\sqrt{\eta^{2}L^{2}\left(1-c\right)^{4}+4(1-c)c}}{2(1-c)}.
\end{align*}
For instance with $c=0.1$ $\eta L=0.1$, a value of $r=1.01$ satisfies
the inequality. Note that when the learning rate is decreased further,
the allowable values of $r$ increase. This suggests that at the point
in which the learning rate would normally decrease by a large factor
such as 10 in a stage-wise schedule, instead the learning rate should
be decreased geometrically, by a factor $\alpha$ each step, until
it reaches the 10x lower value. An example of this kind of schedule
is shown in Figure \ref{fig:gradual-schematic}. Notice that the $\alpha_{k}$
and $\beta_{k}$ values stay very reasonable under this gradual scheme
compared to the other schemes considered so far. This will happen
in a matter of a few epochs for typical problems such as CIFAR-10
training.
\vspace{-0.5em}
\paragraph{An empirical study}
The violation of the inequality that occurs when the learning rate
is changed suddenly is \textbf{not just an artifact of the analysis} used,
a spike in the training loss is readily observed in practice, An example
that occurs during CIFAR-10 training is shown in Figure \ref{fig:gradual_train_loss}. Full details of the experimental setup are available in the Appendix.
The gradual approach avoids the spike seen when the learning rate
is changed suddenly. Although the training loss recovers rapidly after
the spike, the gradual approach quickly obtains a lower training loss.
The gradual approach modifies the standard scheme by increasing $c$
by 10-fold (up to a maximum of 1.0 for $c$) whenever $\eta$ is decreased
10-fold. Instead of an instantaneous change we changed both with a
$1.0005$ geometric factor each step until they reached their new
value. As can be seen in Figure \ref{fig:gradual_test_accuracy},
there is also no loss of final test accuracy at all from using the
gradual schedule for both CIFAR-10 or ImageNet. 
\vspace{-0.5em}
\section*{Conclusion}
\vspace{-0.5em}
Our analysis provides a better understanding of momentum methods for non-convex optimization through the lens of the primal averaging form. We characterize the extra terms introduced introduced into the Lyapunov analysis from the use of momentum, and show when these terms are beneficial and when they are harmful. We also analyze the behavior of the primal averaging form under changing step size schemes, and show the surprising result that standard schemes do not make sense in the averaging form, and suggest alternatives that are better behaved.
\pagebreak
\bibliographystyle{plainnat}
\bibliography{abntex2-options,biblio}

\clearpage

\appendix

\section{SGD+M and SPA equivalence}
\begin{thm}
Define the SGD+M method by the two sequences:
\begin{align*}
m_{k+1} & =\beta_{k}m_{k}+\nabla f(x_{k},\xi_{k}),\\
x_{k+1} & =x_{k}-\alpha_{k}m_{k+1},
\end{align*}
and the SPA sequences as:
\begin{align*}
z_{k+1} & =z_{k}-\eta_{k}\nabla f\left(x_{k},\xi_{k}\right),\\
x_{k+1} & =\left(1-c_{k+1}\right)x_{k}+c_{k+1}z_{k+1}.
\end{align*}
Consider the case where $m_{0}=0$ for SGD+M and and $z_{0}=0$ for
SPA. Then if $c_{1}=\alpha_{0}/\eta_{0}$ and for $k\geq0$
\[
\eta_{k+1}=\frac{\eta_{k}-\alpha_{k}}{\beta_{k+1}},\qquad c_{k+1}=\frac{\alpha_{k}}{\eta_{k}},
\]
The $x$ sequence produced by the $SPA$ method is identical to the
$x$ sequence produced by the SGD+M method.
\end{thm}
\begin{proof}
Consider the base case where $x_{0}=z_{0}$. Then for SGD+M:
\[
m_{1}=\nabla f(x_{0},\xi_{0})
\]
\begin{equation}
\therefore x_{1}=x_{0}-\alpha_{0}\nabla f(x_{0},\xi_{0})\label{eq:sgdm_base_eq}
\end{equation}
and for the SPA form:
\[
z_{1}=x_{0}-\eta_{0}\nabla f(x_{0},\xi_{0})
\]
\begin{align}
x_{1} & =\left(1-c_{0}\right)x_{0}+c_{0}\left(x_{0}-\eta_{0}\nabla f(x_{0},\xi_{0})\right)\nonumber \\
 & =x_{0}-c_{0}\eta_{0}\nabla f(x_{0},\xi_{0})\label{eq:spa_base_eq}
\end{align}
Clearly Equation \ref{eq:sgdm_base_eq} is equivalent to Equation
\ref{eq:spa_base_eq} when $\alpha_{0}=c_{0}\eta_{0}$. 

Now consider $k>0$. We will define $z_{k}$ in term of quantities
in the SGD+M method, then show that with this definition, the step-to-step
changes in $z$ correspond exactly to the SPA method. In particular,
let:
\begin{align}
z_{k} & =x_{k}-\left(\frac{1}{c_{k}}-1\right)\alpha_{k-1}m_{k}.\label{eq:zdef}
\end{align}
Then
\begin{align*}
z_{k+1} & =x_{k+1}-\left(\frac{1}{c_{k+1}}-1\right)\alpha_{k}m_{k+1}\\
 & =x_{k}-\alpha_{k}m_{k+1}-\left(\frac{1}{c_{k+1}}-1\right)\alpha_{k}m_{k+1}\\
 & =z_{k}+\left(\frac{1}{c_{k}}-1\right)\alpha_{k-1}m_{k}-\frac{\alpha_{k}}{c_{k+1}}\left(\beta_{k}m_{k}+\nabla f(x_{k},\xi_{k})\right)\\
 & =z_{k}+\left[\left(\frac{1}{c_{k}}-1\right)\alpha_{k-1}-\frac{\alpha_{k}}{c_{k+1}}\beta_{k}\right]m_{k}-\frac{\alpha_{k}}{c_{k+1}}\nabla f(x_{k},\xi_{k}).
\end{align*}
This is equivalent to the SPA step
\[
z_{k+1}=z_{k}-\eta_{k}\nabla f\left(x_{k},\xi_{k}\right),
\]
 as long as $\frac{\alpha_{k}}{c_{k+1}}=\eta_{k}$ and 
\begin{align*}
0 & =\left(\frac{1}{c_{k}}-1\right)\alpha_{k-1}-\frac{\alpha_{k}}{c_{k+1}}\beta_{k}\\
 & =\left(\eta_{k-1}-\alpha_{k-1}\right)-\eta_{k}\beta_{k},
\end{align*}
\[
\text{i.e. }\eta_{k}=\frac{\eta_{k-1}-\alpha_{k-1}}{\beta_{k}}.
\]
Using this definition of the $z$sequence, we can rewrite the SGD+M
$x$ sequence using a rearrangement of Equation \ref{eq:zdef}:
\begin{align*}
m_{k+1} & =\left(\frac{1}{c_{k+1}}-1\right)^{-1}\alpha_{k}^{-1}\left(x_{k+1}-z_{k+1}\right),\\
 & =\frac{c_{k+1}}{1-c_{k+1}}\alpha_{k}^{-1}\left(x_{k+1}-z_{k+1}\right),
\end{align*}
 as
\begin{align*}
x_{k+1} & =x_{k}-\alpha_{k}m_{k+1}\\
 & =x_{k}-\frac{c_{k+1}}{1-c_{k+1}}\left(x_{k+1}-z_{k+1}\right)\\
 & =x_{k}-\frac{c_{k+1}}{1-c_{k+1}}x_{k+1}+\frac{c_{k+1}}{1-c_{k+1}}z_{k+1}\\
 & =\left(1-c_{k+1}\right)x_{k}+c_{k+1}z_{k+1},
\end{align*}
matching the SPA update.
\end{proof}

\section{Lemmas}
\begin{lem}
(LEMMA 1.2.3, \citet{Nesterov-convex})\label{lem:smooth} Suppose
that $f$ is differentiable and has $L$-Lipschitz gradient:
\begin{equation}
\left\Vert \nabla f(x)-\nabla f(y)\right\Vert \leq L\left\Vert x-y\right\Vert ,\label{eq:lipschitz-gradients}
\end{equation}
then: 
\begin{align}
\left|f(y)-f(x)-\langle\nabla f(x),y-x\rangle\right|\leq\frac{L}{2}\|x-y\|_{2}^{2},\qquad\forall x,y\in\mathbb{R}^{n}.\label{eq:smoothf}
\end{align}
in particular, 
\begin{equation}
f(y)\leq f(x)+\langle\nabla f(x),y-x\rangle+\frac{L}{2}\|x-y\|_{2}^{2},\label{eq:smooth-ub}
\end{equation}
\begin{equation}
\text{and }f(y)\geq f(x)+\langle\nabla f(x),y-x\rangle-\frac{L}{2}\|x-y\|_{2}^{2}.\label{eq:smooth-lb}
\end{equation}
\end{lem}
We will make heavy use of the fact that the $x_{k+1}$ update can
be rearranged to give:
\[
z_{k}=x_{k}-\left(\frac{1}{c_{k}}-1\right)\left(x_{k-1}-x_{k}\right).
\]

\begin{lem}
\label{lem:lemma_xdiff_step}Suppose that $f$ is differentiable and
has $L$-Lipschitz gradient, then the updates of the SPA form obey
for $k\geq1$:
\begin{align*}
\frac{L}{c_{k+1}^{2}}\mathbb{E}\left\Vert x_{k+1}-x_{k}\right\Vert ^{2} & \leq L\left(\frac{1}{c_{k}}-1+\eta_{k}L\right)\left(\frac{1}{c_{k}}-1\right)\left\Vert x_{k}-x_{k-1}\right\Vert ^{2}+\eta_{k}^{2}L\mathbb{E}\left\Vert \nabla f\left(x_{k},\xi_{k}\right)\right\Vert ^{2}\\
 & +2\eta_{k}L\left(\frac{1}{c_{k}}-1\right)\left[f(x_{k-1})-f(x_{k})\right].
\end{align*}
\end{lem}
\begin{proof}
We may write the difference of the $x_{k}$ updates between steps
as:
\[
x_{k+1}-x_{k}=c_{k+1}\left(z_{k}-x_{k}\right)-\eta_{k}c_{k+1}\nabla f\left(x_{k},\xi_{k}\right)
\]
Recall that:
\[
z_{k}-x_{k}=\left(\frac{1}{c_{k}}-1\right)\left(x_{k}-x_{k-1}\right).
\]
So:
\[
x_{k+1}-x_{k}=c_{k+1}\left(\frac{1}{c_{k}}-1\right)\left(x_{k}-x_{k-1}\right)-\eta_{k}c_{k+1}\nabla f\left(x_{k},\xi_{k}\right)
\]
Taking the squared norm and expanding, then taking expectations with
respect to $\xi_{k}$ gives: 
\begin{align*}
\mathbb{E}\left\Vert x_{k+1}-x_{k}\right\Vert ^{2} & =c_{k+1}^{2}\left(\frac{1}{c_{k}}-1\right)^{2}\left\Vert x_{k}-x_{k-1}\right\Vert ^{2}+c_{k+1}^{2}\eta_{k}^{2}\mathbb{E}\left\Vert \nabla f\left(x_{k},\xi_{k}\right)\right\Vert ^{2}\\
 & -2\eta_{k}c_{k+1}^{2}\left(\frac{1}{c_{k}}-1\right)\left\langle \nabla f\left(x_{k}\right),x_{k}-x_{k-1}\right\rangle 
\end{align*}
Now we apply the smoothness lower bound (Eq. \ref{eq:smooth-lb}):
\[
f(x_{k-1})\geq f(x_{k})+\left\langle \nabla f(x_{k}),x_{k-1}-x_{k}\right\rangle -\frac{L}{2}\left\Vert x_{k}-x_{k-1}\right\Vert ^{2}
\]
Rearranged into the form:
\[
-\left\langle \nabla f(x_{k}),x_{k}-x_{k-1}\right\rangle \leq f(x_{k-1})-f(x_{k})+\frac{L}{2}\left\Vert x_{k}-x_{k-1}\right\Vert ^{2}
\]
to give:
\begin{align*}
\mathbb{E}\left\Vert x_{k+1}-x_{k}\right\Vert ^{2} & \leq c_{k+1}^{2}\left(\frac{1}{c_{k}}-1\right)^{2}\left\Vert x_{k}-x_{k-1}\right\Vert ^{2}+c_{k+1}^{2}\eta_{k}^{2}\mathbb{E}\left\Vert \nabla f\left(x_{k},\xi_{k}\right)\right\Vert ^{2}\\
 & + 2\eta_{k}c_{k+1}^{2}\left(\frac{1}{c_{k}}-1\right)\left[f(x_{k-1})-f(x_{k})\right]+\eta_{k}Lc_{k+1}^{2}\left(\frac{1}{c_{k}}-1\right)\left\Vert x_{k}-x_{k-1}\right\Vert ^{2}
\end{align*}
Now group terms and multiply by $L/c_{k+1}^{2}$:
\begin{align*}
\frac{L}{c_{k+1}^{2}}\mathbb{E}\left\Vert x_{k+1}-x_{k}\right\Vert ^{2} & \le L\left(\frac{1}{c_{k}}-1+\eta_{k}L\right)\left(\frac{1}{c_{k}}-1\right)\left\Vert x_{k}-x_{k-1}\right\Vert ^{2}+\eta_{k}^{2}L\mathbb{E}\left\Vert \nabla f\left(x_{k},\xi_{k}\right)\right\Vert ^{2}\\
 & +2\eta_{k}L\left(\frac{1}{c_{k}}-1\right)\left[f(x_{k-1})-f(x_{k})\right]
\end{align*}
\end{proof}\begin{lem}
\label{lem:f_update}Suppose that $f$ is differentiable and has $L$-Lipschitz
gradients, then the updates of the SPA form obey for $k\geq1$:
\begin{align*}
\mathbb{E}\left[f(z_{k+1})\right]+\frac{\eta_{k}}{2}\left\Vert \nabla f(z_{k})\right\Vert ^{2} & \leq f(z_{k})-\frac{\eta_{k}}{2}\left\Vert \nabla f(x_{k})\right\Vert ^{2}+\frac{1}{2}\eta_{k}L^{2}\left(\frac{1}{c_{k}}-1\right)^{2}\left\Vert x_{k}-x_{k-1}\right\Vert ^{2}\\
 & +\frac{1}{2}\eta_{k}^{2}L\mathbb{E}\left[\left\Vert \nabla f(x_{k},\xi_{k})\right\Vert ^{2}\right],
\end{align*}
where the expectation is with respect to $\xi_{k}$, and is conditional
on the iterates and gradients from prior steps. 
\end{lem}
\begin{proof}
Using $z_{k+1}=z_{k}-\eta_{k}\nabla f\left(x_{k},\xi_{k}\right)$
and the smoothness upper bound (Equation \ref{eq:smooth-ub}):
\begin{align*}
\mathbb{E}\left[f(z_{k+1})\right] & \leq f(z_{k})-\eta_{k}\mathbb{E}\left\langle \nabla f(z_{k}),\nabla f\left(x_{k},\xi_{k}\right)\right\rangle +\frac{1}{2}\eta_{k}^{2}L\mathbb{E}\left[\left\Vert \nabla f(x_{k},\xi_{k})\right\Vert ^{2}\right]\\
 & =f(z_{k})-\eta_{k}\left\langle \nabla f(z_{k}),\nabla f(x_{k})\right\rangle +\frac{1}{2}\eta_{k}^{2}L\mathbb{E}\left[\left\Vert \nabla f(x_{k},\xi_{k})\right\Vert ^{2}\right]\\
 & =f(z_{k})+\frac{\eta_{k}}{2}\left\Vert \nabla f(z_{k})-\nabla f(x_{k})\right\Vert ^{2}-\frac{\eta_{k}}{2}\left\Vert \nabla f(z_{k})\right\Vert ^{2}-\frac{\eta_{k}}{2}\left\Vert \nabla f(x_{k})\right\Vert ^{2}\\
 & +\frac{1}{2}\eta_{k}^{2}L\mathbb{E}\left[\left\Vert \nabla f(x_{k},\xi_{k})\right\Vert ^{2}\right]
\end{align*}
Now we use our assumption that the gradients are Lipschitz (Eq. \ref{eq:lipschitz-gradients}):
\[
\left\Vert \nabla f(z_{k})-\nabla f(x_{k})\right\Vert ^{2}\leq L^{2}\left\Vert z_{k}-x_{k}\right\Vert ^{2}=L^{2}\left(\frac{1}{c_{k}}-1\right)^{2}\left\Vert x_{k}-x_{k-1}\right\Vert ^{2}
\]
 to give:
\begin{align*}
\mathbb{E}\left[f(z_{k+1})\right]+\frac{\eta_{k}}{2}\left\Vert \nabla f(z_{k})\right\Vert ^{2} & \leq f(z_{k})-\frac{\eta_{k}}{2}\left\Vert \nabla f(x_{k})\right\Vert ^{2}+\frac{1}{2}\eta^{k}L^{2}\left(\frac{1}{c_{k}}-1\right)^{2}\left\Vert x_{k}-x_{k-1}\right\Vert ^{2}\\
 & +\frac{1}{2}\eta_{k}^{2}L\mathbb{E}\left[\left\Vert \nabla f(x_{k},\xi_{k})\right\Vert ^{2}\right]
\end{align*}
\end{proof}

\section{Building the Lyapunov function}

Consider Lemma \ref{lem:lemma_xdiff_step} after taking expectations
and dividing by $2\eta_{k}^{2}$:
\begin{align*}
 & \frac{L}{2\eta_{k}^{2}c_{k+1}^{2}}\mathbb{E}\left[\left\Vert x_{k+1}-x_{k}\right\Vert ^{2}\right]\\
 & \leq\frac{1}{2\eta_{k}^{2}}L\left(\frac{1}{c_{k}}-1+\eta_{k}L\right)\left(\frac{1}{c_{k}}-1\right)\left\Vert x_{k}-x_{k-1}\right\Vert ^{2}\\
 & +\frac{1}{\eta_{k}}L\left(\frac{1}{c_{k}}-1\right)\left[f(x_{k-1})-f(x_{k})\right]\\
 & +\frac{1}{2}L\mathbb{E}\left[\left\Vert \nabla f\left(x_{k},\xi_{k}\right)\right\Vert ^{2}\right]
\end{align*}
and Lemma \ref{lem:f_update} divided by $\eta_{k}^{2}:$
\begin{flalign*}
 & \frac{1}{\eta_{k}^{2}}\mathbb{E}\left[f(z_{k+1})\right]+\frac{1}{2\eta_{k}}\left\Vert \nabla f(z_{k})\right\Vert ^{2}\\
 & \leq\frac{1}{\eta_{k}^{2}}f(z_{k})-\frac{1}{2\eta_{k}}\left\Vert \nabla f(x_{k})\right\Vert ^{2}\\
 & +\frac{1}{2}\frac{1}{\eta_{k}}L^{2}\left(\frac{1}{c_{k}}-1\right)^{2}\left\Vert x_{k}-x_{k-1}\right\Vert ^{2}\\
 & +\frac{1}{2}L\mathbb{E}\left[\left\Vert \nabla f(x_{k},\xi_{k})\right\Vert ^{2}\right]
\end{flalign*}
Combining those bounds results in the following natural choice of
Lyapunov function $\Lambda$:
\begin{align}
\Lambda_{k+1} & =\frac{1}{\eta_{k}^{2}}\left[f(z_{k+1})-f_{*}\right]\label{eq:nonstatic-lya-1}\\
 & +\frac{L}{\eta_{k}}\left(\frac{1}{c_{k}}-1\right)\left[f(x_{k})-f_{*}\right]\\
 & +\frac{1}{2}L\frac{1}{\eta_{k}^{2}c_{k+1}^{2}}\left\Vert x_{k+1}-x_{k}\right\Vert ^{2}
\end{align}
and yields the bound for $k\geq1$:
\begin{align}
 & \frac{1}{2\eta_{k}}\left\Vert \nabla f(x_{k})\right\Vert ^{2}+\frac{1}{2\eta_{k}}\left\Vert \nabla f(z_{k})\right\Vert ^{2}\nonumber \\
 & \leq\Lambda_{k}-\mathbb{E}\left[\Lambda_{k+1}\right]+L\mathbb{E}\left[\left\Vert \nabla f(x_{k},\xi_{k})\right\Vert ^{2}\right]\nonumber \\
 & +\frac{L}{2}\left[\frac{1}{\eta_{k}^{2}}\left(\frac{1}{c_{k}}-1+\eta_{k}L\right)\left(\frac{1}{c_{k}}-1\right)+\frac{1}{\eta_{k}}L\left(\frac{1}{c_{k}}-1\right)^{2}-\frac{1}{\eta_{k-1}^{2}c_{k}^{2}}\right]\left\Vert x_{k}-x_{k-1}\right\Vert ^{2}\nonumber \\
 & +\left[\frac{1}{\eta_{k}}L\left(\frac{1}{c_{k}}-1\right)-\frac{1}{\eta_{k-1}}L\left(\frac{1}{c_{k-1}}-1\right)\right]\left[f(x_{k-1})-f_{*}\right]\\
 & +\left[\frac{1}{\eta_{k}^{2}}-\frac{1}{\eta_{k-1}^{2}}\right]\left[f(z_{k})-f_{*}\right]
\end{align}

\section{Telescoping}

In order to complete a convergence rate proof, we must consider the
behavior of the method at step $0$. The above two lemmas are simplified
in this case, yielding the following bound replacing Lemma \ref{lem:lemma_xdiff_step}:
\begin{align*}
\mathbb{E}\left[\left\Vert x_{1}-x_{0}\right\Vert ^{2}\right] & =c_{1}^{2}\eta_{0}^{2}\mathbb{E}\left[\left\Vert \nabla f\left(x_{0},\xi_{0}\right)\right\Vert ^{2}\right],
\end{align*}
and replacing Lemma \ref{lem:f_update}
\begin{align*}
\mathbb{E}\left[f(z_{1})\right]+\frac{1}{2}\eta_{0}\left\Vert \nabla f(z_{0})\right\Vert ^{2} & \leq f(z_{0})-\frac{1}{2}\eta_{0}\left\Vert \nabla f(x_{0})\right\Vert ^{2}+\frac{1}{2}\eta_{0}^{2}L\mathbb{E}\left[\left\Vert \nabla f(x_{0},\xi_{0})\right\Vert ^{2}\right].
\end{align*}
Multiplying the first result by $L/(2c_{1}^{2}\eta_{0}^{2})$ and
dividing the second result by $\eta_{0}$ , we may sum these equations
to give:
\begin{align*}
\frac{1}{2\eta_{0}}\left\Vert \nabla f(x_{0})\right\Vert ^{2}+\frac{1}{2\eta_{0}}\left\Vert \nabla f(z_{0})\right\Vert ^{2} & \leq\frac{1}{\eta_{0}^{2}}\left[f(z_{0})-f_{*}\right]-\frac{1}{\eta_{0}^{2}}\mathbb{E}\left[f(z_{1})-f_{*}\right]\\
 & -\frac{L}{2\eta_{0}^{2}c_{1}^{2}}\mathbb{E}\left[\left\Vert x_{1}-x_{0}\right\Vert ^{2}\right]+L\mathbb{E}\left\Vert \nabla f(x_{0},\xi_{0})\right\Vert ^{2}.
\end{align*}
Now consider the behavior of the SGD+M method when we use a fixed
step size $\eta$. As long as
\[
\eta\leq\frac{2-c}{Lc(1-c)},
\]
 and $E\left[\left\Vert \nabla f(x_{k},\xi_{k})\right\Vert ^{2}\right]\leq G^{2}$
, we may telescope from this base case to step $T$, yielding:
\begin{align*}
 & \frac{1}{\eta}\sum_{k}^{T}\mathbb{E}\left[\frac{1}{2}\left\Vert \nabla f(x_{k})\right\Vert ^{2}+\frac{1}{2}\left\Vert \nabla f(z_{k})\right\Vert ^{2}\right]\\
 & \leq\frac{1}{\eta^{2}}\left[f(z_{0})-f_{*}\right]+\frac{L}{\eta}\left(\frac{1}{c_{1}}-1\right)\left[f(x_{0})-f_{*}\right]+TLG^{2}.
\end{align*}
Multiplying by $\eta/T$ gives a bound on the average iterate:
\begin{align*}
 & \frac{1}{T}\sum_{k}^{T}\mathbb{E}\left[\frac{1}{2}\left\Vert \nabla f(x_{k})\right\Vert ^{2}+\frac{1}{2}\left\Vert \nabla f(z_{k})\right\Vert ^{2}\right]\\
 & \leq\frac{1}{\eta T}\left[f(z_{0})-f_{*}\right]+\frac{L}{T}\left(\frac{1}{c_{1}}-1\right)\left[f(x_{0})-f_{*}\right]+\eta LG^{2}.
\end{align*}
Using the optimal step size $\eta^{2}=T^{-1}L^{-1}G^{-2}\left[f(z_{0})-f_{*}\right]$
gives:
\[
\frac{1}{T}\sum_{k}^{T}\mathbb{E}\left[\frac{1}{2}\left\Vert \nabla f(x_{k})\right\Vert ^{2}+\frac{1}{2}\left\Vert \nabla f(z_{k})\right\Vert ^{2}\right]\leq2G\frac{\sqrt{L\left[f(z_{0})-f_{*}\right]}}{\sqrt{T}}+\frac{L}{T}\left(\frac{1}{c_{1}}-1\right)\left[f(x_{0})-f_{*}\right],
\]
whereas the more realistic step size $\eta^{2}=T^{-1}L^{-2}$ gives
\begin{align*}
 & \frac{1}{T}\sum_{k}^{T}\mathbb{E}\left[\frac{1}{2}\left\Vert \nabla f(x_{k})\right\Vert ^{2}+\frac{1}{2}\left\Vert \nabla f(z_{k})\right\Vert ^{2}\right]\\
 & \leq\frac{L}{\sqrt{T}}\left[f(z_{0})-f_{*}\right]+\frac{L}{T}\left(\frac{1}{c_{1}}-1\right)\left[f(x_{0})-f_{*}\right]+\frac{G^{2}}{\sqrt{T}}.
\end{align*}
In each case, the extra term $\left[f(x_{0})-f_{*}\right]$ that differs
from the standard SGD Lyapunov function decays at a 1/T rate, and
so becomes negligible for large $T$.

\subsection{Removing the bounded gradients assumption}

The above argument uses a bounded gradients assumption, however this
assumption can be removed by moving a small part of the $\left\Vert \nabla f(x_{k})\right\Vert ^{2}$
term from the left to the right hand side of the Lyapunov step equation,
so that we can use $\mathbb{E}\left[\left\Vert \nabla f(x_{k},\xi_{k})\right\Vert ^{2}\right]-\left\Vert \nabla f(x_{k})\right\Vert ^{2}=\mathbb{E}\left[\left\Vert \nabla f(x_{k},\xi_{k})-\nabla f(x_{k})\right\Vert ^{2}\right]$.
The final convergence rate then depends instead on 
\[
\sigma^{2}=\mathbb{E}\left[\left\Vert \nabla f(x_{k},\xi_{k})-\nabla f(x_{k})\right\Vert ^{2}\right].
\]
The fraction to move depends on the final step size, and for $\eta\propto1/\sqrt{T}$
it doesn't significantly effect the final convergence rate.

\section{SGD reference proof}

We reproduce the standard argument for non-convex SGD convergence
here for easy comparison to our SGD+M proof above. Consider the step
$x_{k+1}=x_{k}-\eta_{k}\nabla f(x_{k},\xi_{k}).$ Then:
\begin{align*}
f(x_{k+1}) & \leq f(x_{k})+\left\langle \nabla f(x_{k}),x_{k+1}-x_{k}\right\rangle +\frac{1}{2}L\left\Vert x_{k+1}-x_{k}\right\Vert ^{2}\\
 & =f(x_{k})-\eta_{k}\left\langle \nabla f(x_{k}),\nabla f(x_{k},\xi_{k})\right\rangle +\frac{1}{2}L\eta_{k}^{2}\left\Vert \nabla f(x_{k},\xi_{k})\right\Vert ^{2}.
\end{align*}
Taking expectations and using the bounded gradients assumption gives:
\[
\mathbb{E}\left[f(x_{k+1})\right]\leq f(x_{k})-\eta_{k}\left\Vert \nabla f(x_{k})\right\Vert ^{2}+\frac{1}{2}L\eta_{k}^{2}G^{2}.
\]
Define $\Lambda_{k}=\eta_{k}^{-2}\mathbb{E}\left[f(x_{k})-f_{*}\right]$:
Then rearranging gives:
\[
\frac{1}{\eta_{k}}\mathbb{E}\left[\left\Vert \nabla f(x_{k})\right\Vert ^{2}\right]\leq\Lambda_{k}-\mathbb{E}\left[\Lambda_{k+1}\right]+\frac{1}{2}LG^{2}+(\eta_{k}^{-2}-\eta_{k-1}^{-2})\left[f(z_{k})-f_{*}\right].
\]
Assuming a fixed step size, we telescope from $0$ to $T$ after taking
total expectations:
\[
\frac{1}{\eta}\sum_{k=0}^{T}\mathbb{E}\left[\left\Vert \nabla f(x_{k})\right\Vert ^{2}\right]\leq\Lambda_{0}-\mathbb{E}\left[\Lambda_{T+1}\right]+\frac{1}{2}LG^{2}T.
\]
So:
\[
\frac{1}{T}\sum_{k=0}^{T}\mathbb{E}\left[\left\Vert \nabla f(x_{k})\right\Vert ^{2}\right]\leq\frac{1}{\eta T}\left[f(x_{0})-f_{*}\right]+\frac{1}{2}L\eta G^{2},
\]
using the optimal step size 
\[
\eta=\sqrt{\frac{2\left[f(x_{0})-f_{*}\right]}{TLG^{2}}}
\]
gives:
\[
\frac{1}{T}\sum_{k=0}^{T}\mathbb{E}\left[\left\Vert \nabla f(x_{k})\right\Vert ^{2}\right]\leq\frac{G\sqrt{2L\left[f(x_{0})-f_{*}\right]}}{\sqrt{T}},
\]
which for large $T$, only differs from the SGD+M rate by a factor
$\sqrt{2}$.

\section{Augmented Lyapunov}

In Section \ref{subsec:noise-cancelation}, we consider the case of
constant $\eta$ and $c$, and we introduce the additional assumption
that $\left\langle \nabla f\left(x_{k-1}\right),x_{k-1}-x_{k-2}\right\rangle =0$,
so that:
\begin{align}
\left\Vert x_{k}-x_{k-1}\right\Vert ^{2} & =\left(1-c\right)^{2}\left\Vert x_{k-1}-x_{k-2}\right\Vert ^{2}+c^{2}\eta^{2}\left\Vert \nabla f\left(x_{k-1},\xi_{k-1}\right)\right\Vert ^{2},\label{eq:simplifed_recursion}
\end{align}
We want to modify the Lyapunov function so that we have:
\[
\rho\varGamma_{k+1}\leq\rho\varGamma_{k}+\rho_{k}c^{2}\eta^{2}\left\Vert \nabla f\left(x_{k},\xi_{k}\right)\right\Vert ^{2},
\]
where $\rho$ is a negative, and $\varGamma_{k+1}=\left\Vert x_{k+1}-x_{k}\right\Vert ^{2}$.
Consider the constants in front of the $\left\Vert x_{k}-x_{k-1}\right\Vert ^{2}$
term in the Lyapunov step:
\[
\frac{L}{2c\eta}\left[-\frac{2-c}{\eta}+\frac{L-Lc}{c}\right]\left\Vert x_{k+1}-x_{k}\right\Vert ^{2}.
\]
Using this expression, clearly our requirement on $\rho$ will be
satisfied if:
\[
\rho\left(1-c\right)^{2}+\frac{L}{2c\eta}\left[-\frac{2-c}{\eta}+\frac{L-Lc}{c}\right]=\rho,
\]
solving for $\rho$ gives:
\[
\rho=\frac{L}{2\eta c^{2}}\left[\frac{L\left(1-c\right)}{c\left(2-c\right)}-\frac{1}{\eta}\right].
\]
$\rho$ will be negative when:
\[
\eta\leq\frac{c(2-c)}{L(1-c)},
\]
which covers all reasonable choices of hyper-parameters as considered
in the convergence rate theory above. Using this $\rho$, we have
an additional term in the Lyapunov step equation given by weighting
the gradient noise term in Eq. \ref{eq:simplifed_recursion} by $\rho$:
\[
\rho c^{2}\eta^{2}\left\Vert \nabla f\left(x_{k},\xi_{k}\right)\right\Vert ^{2}=\left[\frac{\eta L\left(1-c\right)}{c\left(2-c\right)}-1\right]\frac{L}{2}\left\Vert \nabla f\left(x_{k},\xi_{k}\right)\right\Vert ^{2}.
\]
This value is very close to $-\frac{L}{2}\left\Vert \nabla f\left(x_{k},\xi_{k}\right)\right\Vert ^{2}$
for sensible hyper-parameter values. For instance, for a typical $\eta=T^{-1/2}L^{-1}$
choice you get for the inner term:
\[
\frac{\eta L\left(1-c\right)}{c\left(2-c\right)}-1=\frac{1-c}{\sqrt{T}c\left(2-c\right)}-1,
\]
which for $c=0.1$ and $T=10,000$, yields $\frac{1-c}{\sqrt{T}c\left(2-c\right)}-1=0.047-1.$

\section{Details of experiments}

In both cases below, when expressed in SPA form, the initial LR 0.1 corresponds to an initial learning rate of 1.0 and $c=0.1$.

\subsection*{CIFAR10}

Our data augmentation pipeline consisted of random horizontal flipping,
then random crop to 32x32, then normalization by centering around
(0.5, 0.5, 0.5). We used the standard learning rate schedule for this
problem, consisting of a 10-fold decrease at epochs 150 and 225. Test/train/validate splits are standard. Total running time is < 24 hours per run. Our results are averaged over 20 seeds for each variant.

\begin{tabular}{|c|c|}
\hline 
Hyper-parameter  & Value\tabularnewline
\hline 
\hline 
Architecture  & PreAct ResNet152\tabularnewline
\hline 
Epochs  & 300\tabularnewline
\hline 
GPUs  & 1xV100\tabularnewline
\hline 
Batch Size per GPU  & 128\tabularnewline
\hline 
Decay  & 0.0001\tabularnewline
\hline 
\end{tabular}

\subsection*{ImageNet}

Data augmentation consisted of the RandomResizedCrop(224) operation
in PyTorch, followed by RandomHorizontalFlip then normalization to
mean={[}0.485, 0.456, 0.406{]} and std={[}0.229, 0.224, 0.225{]}.
We used the standard learning rate schedule for this problem, where
the learning rate is decreased 10 fold every 30 epochs. Test/train/validate splits are standard. Total running time is < 24 hours per run. Our results are averaged over 5 seeds for each variant.

\begin{tabular}{|c|c|}
\hline 
Hyper-parameter  & Value\tabularnewline
\hline 
\hline 
Architecture  & ResNet50\tabularnewline
\hline 
Epochs  & 100\tabularnewline
\hline 
GPUs  & 8xV100\tabularnewline
\hline 
Batch size per GPU  & 32\tabularnewline
\hline 
Decay  & 0.0001\tabularnewline
\hline 
\end{tabular}
\end{document}